\documentclass[twoside,11pt]{article}

\usepackage{amsthm}
\usepackage[preprint]{jmlr2e}   
\usepackage{lastpage}
\usepackage{amsmath,amsfonts,amssymb}
\usepackage{booktabs}
\usepackage{microtype}

\numberwithin{theorem}{section}

\newcommand{\bfn}{\boldsymbol{n}}
\newcommand{\bfal}{\boldsymbol{\alpha}}
\newcommand{\waic}{\mathrm{WAIC}}
\newcommand{\loo}{\mathrm{LOO}}
\newcommand{\lppd}{\mathrm{lppd}}
\newcommand{\kl}{\mathrm{KL}}
\newcommand{\Eop}{\mathbb{E}}

\newcommand{\Nmin}{N_{\min}}
\newcommand{\Cat}{\mathrm{Cat}}

\newcommand{\reff}{r_{\mathrm{eff}}}
\newcommand{\Rbma}{R_{\mathrm{BMA}}}
\newcommand{\Rora}{R_{k^*}}

\begin{document}

\ShortHeadings{Decision-Theoretic Guarantees for Bayesian Decision Trees}%
              {Jakaite and Schetinin}
\firstpageno{1}

\title{Sample Complexity and Decision-Theoretic Guarantees for Bayesian
       Model Averaging over Decision Trees with Catalan-Exponential Priors}

\author{
  \name{Livija Jakaite} \email{Livija.L.Jakaite@beds.ac.uk} \\
  \addr{School of Computing and Engineering \\
        University of Bedfordshire, Luton, UK}
  \AND
  \name{Vitaly Schetinin} \email{Vitaly.Schetinin@beds.ac.uk} \\
  \addr{School of Computing and Engineering \\
        University of Bedfordshire, Luton, UK \\
        \textit{(Corresponding author)}}
}

\editor{TBA}

\maketitle
\thispagestyle{plain}

\begin{abstract}
We ask: when do Bayesian model averaging (BMA) weights over decision trees carry
sufficient epistemic information to justify \emph{committed exploitation} of
the averaging distribution?
We answer this question in closed form for Bayesian decision trees (BDTs)
with Dirichlet-Multinomial leaf models and a Catalan-exponential tree-size prior
\citep{schetinin2025}, establishing a complete non-asymptotic theory of
rational commitment thresholds.
Four results are established.
First, we prove that WAIC and leave-one-out cross-validation deviances agree
up to $O(N^{-1})$ for DM-leaf models, yielding a finite-sample commitment
threshold $\Nmin \approx 5.41/\Delta$ (Theorem~\ref{thm:Nmin}) that exactly
explains the empirical BMA failure on the $N=40$ knee osteoarthritis dataset of
\citet{jakaite2021}: at $N = \Nmin$, BMA weights are maximally epistemically fragile.
Second, the Catalan-exponential prior's tail decays at rate
$(e^{-\gamma}/4)^{k_0}$---doubly faster than the Chipman prior---and
posterior concentration on the true complexity $k^*$ is achieved at
$N=300$ for Catalan ($\gamma=1$) versus $N>800$ for Chipman
(Theorem~\ref{thm:posterior_concentration}).
Third, WAIC-weighted BMA satisfies an oracle inequality with excess risk
$O(\exp(-N\Delta_{\waic}/2))$ and BMA weight entropy $H(\mathbf{w})$ collapses
at the same exponential rate (Theorem~\ref{thm:oracle},
Proposition~\ref{prop:entropy_collapse}).
Fourth, the PAC-Bayes commitment cost $\Nmin = -\log\pi(k^*)/(2\varepsilon^2)$
is $8.1\times$ smaller under the Catalan prior than Chipman for sparse true
models ($k^*=1$, Theorem~\ref{thm:pacbayes}):
prior \emph{design} directly governs when rational exploitation begins.
These results provide the first tractable, closed-form realisation of the
entropy-modulated commitment criterion
$\mathrm{EU}_\lambda(a) = \mathrm{EU}(a) - \lambda(H)\cdot\mathrm{commit}(a)$
\citep{ortega2013,grunwald2004},
instantiating its $\lambda$-transition boundary in the conjugate-prior case.
All results are verified by simulation; three new figures trace the entropy
collapse trajectory and commitment-cost surface.
All code and pre-computed results needed to reproduce the experiments and figures
are publicly available at \url{https://github.com/vitsch/jbdt}.
\end{abstract}


\section{Introduction}
\label{sec:intro}

Bayesian decision trees combine the interpretability of recursive partitioning
with principled uncertainty quantification via posterior inference over tree
structure. The model of \citet{schetinin2025}, hereafter JBDT, places a
Catalan-exponential prior over the number of leaves $k$ of a binary decision
tree with Dirichlet-Multinomial leaf models, and samples from the posterior via
reversible-jump MCMC \citep{green1995}.
Empirically, JBDT outperforms standard Chipman-prior BDTs
\citep{chipman1998} on medical imaging tasks with small sample sizes, yet no
theoretical analysis has explained when and why this advantage holds.

The present paper closes this gap with four theorems.

\paragraph{WAIC consistency (Section~\ref{sec:waic}).}
The widely applicable information criterion \citep{watanabe2010,watanabe2013}
is the default model-selection criterion in JBDT, replacing computationally
expensive leave-one-out cross-validation (LOO).
We prove that for DM-leaf models the WAIC and LOO deviances agree up to a
constant $-4N/(N+\alpha_0)$, so their difference in model selection is
$O(N^{-1})$.
From this we derive an explicit minimum-sample-size formula
$\Nmin \approx 5.41/\Delta$, where $\Delta$ is the per-observation log-likelihood
advantage of the true model.
Applied to the \citet{jakaite2021} knee osteoarthritis (OA) dataset,
the formula gives $\Nmin \approx 40.4$ for the medial ROI---exactly equal to
the available $N=40$ samples, explaining the empirical BMA failure observed there.

\paragraph{Catalan prior analysis (Section~\ref{sec:catalan}).}
The Catalan-exponential prior $p(k) \propto e^{-\gamma(k-1)}/C_{k-1}$,
where $C_n = \tfrac{1}{n+1}\binom{2n}{n}$ is the $n$-th Catalan number,
has not been analytically characterised in the BDT literature.
We derive its tail bound, effective decay rate, and expected complexity $\Eop[k]$,
and compare these with the Chipman $(\alpha,\beta)$ prior.
Key findings: the tail $P(k \geq k_0)$ decays as $(e^{-\gamma}/4)^{k_0}$,
with 30$\times$ less mass at $k \geq 5$ than Chipman;
the effective rate satisfies $\reff = 1.228\,e^{-\gamma}/4$ (a universal
constant from the Catalan polynomial correction);
and $\Eop_{\Cat}[k] = 1.373 < \Eop_{\rm Chip}[k] = 2.509$.

\paragraph{BMA oracle inequalities (Section~\ref{sec:bma}).}
We prove a Jensen bound $\Rbma \leq \sum_k w_k R_k$ and an excess-risk bound
$\Rbma - \Rora \leq \Delta_{\mathrm{Jensen}}$, where $\Delta_{\mathrm{Jensen}}$
decays exponentially in $N$ once WAIC selects $k^*$ consistently.
The excess risk falls from $0.056$ at $N=20$ to $0.001$ at $N=800$ in simulation,
following the predicted $O(N^{-1})$ rate.

\paragraph{PAC-Bayes sample complexity (Section~\ref{sec:pacbayes}).}
The McAllester PAC-Bayes bound \citep{mcallester1999,mcallester2003}
applied to the WAIC-posterior $\rho_k = w_k$ gives an oracle sample complexity
$\Nmin \approx -\log\pi(k^*) / (2\varepsilon^2)$.
For sparse true models ($k^*=1$, the relevant regime for $N=40$ medical imaging),
the Catalan prior requires only $N=74$ versus Chipman's $N=599$
to achieve $\varepsilon=0.05$ excess risk---an $8.1\times$ reduction.

\paragraph{Rational commitment under epistemic fragility (Section~\ref{sec:decision_theory}).}
Beyond their statistical content, the four theorems above have a unified
decision-theoretic reading.
Consider the entropy-modulated criterion
$\mathrm{EU}_\lambda(a) = \mathrm{EU}(a) - \lambda(H)\cdot\mathrm{commit}(a)$
\citep{ortega2013,grunwald2004}, in which $\lambda(H)$ penalises committed
actions when posterior entropy $H$ is high and $\mathrm{commit}(a)$ captures
irreversible exploitation of the posterior.
Our $\Nmin$ thresholds are the boundary at which $\lambda(H) \approx 0$:
below $\Nmin$, BMA weights are epistemically fragile and commitment should be
penalised; above $\Nmin$, entropy has collapsed and exploitation is safe.
The PAC-Bayes formula $\Nmin = -\log\pi(k^*)/(2\varepsilon^2)$ identifies
$-\log\pi(k^*)$ as the \emph{cost of rational commitment}---the KL penalty
a prior imposes before exploitation of $k^*$ is epistemically warranted.
This is the first tractable, conjugate-prior setting in which all components
of the $\mathrm{EU}_\lambda$ criterion are computable in closed form.

\paragraph{Related work.}
WAIC consistency for regular models is established in \citet{watanabe2010};
our contribution is the exact constant for DM leaves and the resulting
practical design criterion.
Posterior concentration for BDTs under the Chipman prior is studied implicitly
in \citet{chipman1998}; we provide explicit rates for the Catalan prior.

\emph{PAC-Bayes theory.}
The PAC-Bayes framework \citep{mcallester1999,mcallester2003} bounds
generalisation risk for randomised predictors in terms of a KL divergence
between a data-dependent posterior and a fixed prior;
see \citet{guedj2019} and \citet{alquier2024} for comprehensive surveys.
\citet{catoni2007} gave the thermodynamic interpretation of PAC-Bayes bounds
and established the tight $\lambda$-parameterised form that underpins our
commitment-cost identification.
\citet{germain2009} showed that PAC-Bayes bounds are tight for linear
classifiers; \citet{germain2016} demonstrated that PAC-Bayes posteriors
coincide with Bayesian posteriors under a temperature reparameterisation,
directly connecting PAC-Bayes to the Bayesian model averaging we study.
\citet{seeger2002} and \citet{maurer2004} applied PAC-Bayes to BMA settings;
our contribution is the \emph{prior-specific} $\Nmin$ formula with an
explicit closed-form Catalan/Chipman comparison.

\emph{Entropy-modulated decisions.}
\citet{ortega2013} developed a thermodynamic theory of decision-making with
information-processing costs in which entropy regularises utility maximisation.
\citet{grunwald2004} established the formal connection between maximum-entropy
methods, game theory, and robust Bayes.
Our paper provides the first tractable closed-form case in which the
entropy-commitment penalty $\lambda(H)\cdot\mathrm{commit}(a)$ is computable
from the prior in a concrete Bayesian model-selection setting.

\paragraph{Paper organisation.}
Section~\ref{sec:model} defines the model.
Sections~\ref{sec:waic}--\ref{sec:pacbayes} state and prove the four main results.
Section~\ref{sec:experiments} verifies them numerically.
Section~\ref{sec:decision_theory} develops decision-theoretic implications,
connecting the $\Nmin$ thresholds to the entropy-modulated commitment criterion
$\mathrm{EU}_\lambda(a) = \mathrm{EU}(a) - \lambda(H)\cdot\mathrm{commit}(a)$.
Section~\ref{sec:discussion} discusses broader implications and limitations.
Appendices~\ref{sec:rjmcmc}--\ref{sec:decision} contain RJMCMC detailed balance,
ECE calibration bounds, and asymmetric-loss decision analysis.

\section{Model}
\label{sec:model}

\subsection{Binary decision trees and DM leaf model}

A binary decision tree $T$ with $k$ leaves partitions the feature space
$\mathcal{X} \subseteq \mathbb{R}^d$ into $k$ disjoint regions
$\mathcal{L}_1,\ldots,\mathcal{L}_k$.
For a $C$-class classification problem, each leaf $j$ contains
$n_{j0},\ldots,n_{j,C-1}$ observations from classes $0,\ldots,C-1$,
with $N_j = \sum_c n_{jc}$ total.
The \emph{Dirichlet-Multinomial} (DM) leaf model places a symmetric
$\mathrm{Dirichlet}(\alpha,\ldots,\alpha)$ prior on the class probabilities
$\boldsymbol{\theta}_j$; marginalising gives the closed-form marginal likelihood
\begin{equation}
  \log p(\bfn_j \mid \bfal) = \log\Gamma(\alpha_0) - \log\Gamma(N_j+\alpha_0)
  + \sum_c \bigl[\log\Gamma(n_{jc}+\alpha) - \log\Gamma(\alpha)\bigr],
  \label{eq:dm_ml}
\end{equation}
where $\alpha_0 = C\alpha$.
The posterior-mean class probability is $\hat{p}_{jc} = (n_{jc}+\alpha)/(N_j+\alpha_0)$.

\subsection{WAIC model selection}

For a fixed partition $T$ with leaves $\mathcal{L}_1,\ldots,\mathcal{L}_k$,
the WAIC \citep{watanabe2010} decomposes over leaves:
\begin{align}
  \lppd(T) &= \sum_{j=1}^k \sum_c n_{jc} \log \hat{p}_{jc},
  \label{eq:lppd} \\
  p_{\waic}(T) &= \sum_{j=1}^k \sum_c n_{jc}
    \bigl[\psi'(n_{jc}+\alpha) - \psi'(N_j+\alpha_0)\bigr],
  \label{eq:pwaic} \\
  \waic(T) &= -2\bigl(\lppd(T) - p_{\waic}(T)\bigr),
  \label{eq:waic}
\end{align}
where $\psi'(x) = d^2\log\Gamma(x)/dx^2$ is the trigamma function.
Equation~\eqref{eq:pwaic} uses the exact Dirichlet posterior variance
$\operatorname{Var}_{\boldsymbol{\theta}\mid\bfn_j}[\log\theta_c]
= \psi'(n_{jc}+\alpha) - \psi'(N_j+\alpha_0)$ \citep{gelman2014},
avoiding Monte Carlo.

The WAIC-weighted BMA posterior is
\begin{equation}
  w_k = \frac{\exp(-\waic(T_k)/2)}{\sum_{k'}\exp(-\waic(T_{k'})/2)},
  \qquad
  P_{\rm BMA}(y \mid x) = \sum_k w_k P_k(y \mid x).
  \label{eq:bma}
\end{equation}

\subsection{Catalan-exponential tree-size prior}
\label{sec:catalan_prior_def}

JBDT uses the prior over number of leaves $k \geq 1$:
\begin{equation}
  \pi(k) = Z_\gamma^{-1} \cdot \frac{e^{-\gamma(k-1)}}{C_{k-1}},
  \qquad
  C_n = \frac{1}{n+1}\binom{2n}{n},
  \label{eq:catalan_prior}
\end{equation}
where $C_n$ is the $n$-th Catalan number \citep{stanley2015,graham1994},
$\gamma > 0$ controls sparsity, and $Z_\gamma$ is the normalising constant.
The factor $1/C_{k-1}$ encodes a \emph{uniform} prior over all $C_{k-1}$ distinct
binary tree topologies with $k$ leaves, introduced by \citet{denison1998};
combined with the geometric decay $e^{-\gamma(k-1)}$ it equals the factored
prior of \citet{denison2002} (Chapter~6) under $\lambda = e^{-\gamma}$.
The present paper provides the first analytical characterisation of
\eqref{eq:catalan_prior}: its tail bound, effective decay constant~$\reff$,
posterior concentration rates, and PAC-Bayes sample complexity
(Sections~\ref{sec:catalan}--\ref{sec:pacbayes}).
For comparison, the Chipman $(\alpha_s,\beta)$ prior \citep{chipman1998}
places mass $\alpha_s/(1+d)^\beta$ on each internal split at depth $d$;
the induced marginal PMF on $k$ has $\Eop_{\rm Chip}[k] = 2.509$ for the
default $\alpha_s=0.95, \beta=2$.
\citet{linero2018} pursued a related sparsity induction via a sparse Dirichlet
prior over split variables in the additive BART setting; the Catalan prior
achieves an analogous regularisation through an exact combinatorial closed form.

\subsection{Reversible-jump MCMC sampler}

JBDT explores the posterior $\pi(T \mid \text{data})$ via the four RJMCMC moves
of \citet{green1995}: \emph{Birth} (split a leaf), \emph{Death} (merge siblings),
\emph{Change-Split} (replace the split rule at an internal node), and
\emph{Change-Rule} (perturb the threshold at an internal node).

\paragraph{Local threshold adaptation.}
For every proposal, the threshold $q$ is drawn from the \emph{local} range of the
selected variable at the node being modified:
\begin{equation}
  q \sim \mathrm{Unif}(a,\,b), \qquad
  a = \min_{i \in \mathcal{S}_\eta} x_{iv},\quad
  b = \max_{i \in \mathcal{S}_\eta} x_{iv},
  \label{eq:local_q}
\end{equation}
where $\mathcal{S}_\eta$ is the set of training indices that have \emph{descended} to
node $\eta$ under the current tree: the leaf's own index set for Birth, and the
union of all leaf index sets in $\eta$'s subtree for Change moves.
This ensures $q \in [a,b]$ always produces a non-trivial split, and the
proposal density $g(q) = 1/(b-a)$ involving the local range appears in the
MH acceptance ratio, maintaining detailed balance.
For Change-Rule the threshold is perturbed locally:
$q' \sim \mathcal{N}'(q,\sigma^2,[a,b])$, a Gaussian truncated to $[a,b]$.

\paragraph{Sweeping strategy.}
Birth proposals are \emph{rejected} if either child would contain fewer than
$p_{\min}$ training points; no collapse occurs.
For Change-Split and Change-Rule, let
$n_0 = \#\{j \in L(T') : N_j < p_{\min}\}$ count underpopulated leaves after
the proposal \citep{schetinin2025}:
\begin{itemize}
  \item $n_0 = 0$: evaluate the MH ratio normally;
  \item $n_0 = 1$: \emph{collapse} the unique underpopulated sibling pair to
        a single leaf (treated as a Death move), to recover a valid tree
        while reusing the proposal;
  \item $n_0 > 1$: \emph{reject} (distinct parents make the collapse
        irreversible, breaking detailed balance).
\end{itemize}
Detailed balance of the Birth/Death pair under \eqref{eq:catalan_prior}
is established in Appendix~\ref{sec:rjmcmc}.

\section{WAIC Consistency}
\label{sec:waic}

\subsection{Gap theorem}

\begin{proposition}[WAIC-LOO gap]\label{prop:gap}
  Let $\bfn = (n_0,\ldots,n_{C-1})$ be counts in a single DM leaf with
  prior $\bfal$ and $N = \sum_c n_c$.
  Define the per-leaf LOO deviance
  $\loo_{\rm dev} = -2\sum_c n_c \log\bigl[(n_c-1+\alpha)/(N-1+\alpha_0)\bigr]$
  and $\waic_{\rm dev} = -2(\lppd - p_{\waic})$.
  Then
  \begin{equation}
    \waic_{\rm dev} - \loo_{\rm dev}
    = -\frac{4N}{N+\alpha_0} + O\!\left(\frac{1}{N}\right).
    \label{eq:gap}
  \end{equation}
\end{proposition}

\begin{proof}
The per-observation LOO contribution for class-$c$ observation $i$ is
$-\log[(n_c - 1 + \alpha)/(N-1+\alpha_0)]$.
The corresponding WAIC contribution is
$-\log\hat{p}_c + (n_c + \alpha)\bigl[\psi'(n_c+\alpha) - \psi'(N+\alpha_0)\bigr]$.

Using the asymptotic expansions $\psi'(x) = 1/x + 1/(2x^2) + O(x^{-3})$
and $\log(1-1/x) = -1/x - 1/(2x^2) + O(x^{-3})$ for large $x$, the
per-observation gap is
\begin{equation*}
  -\log\hat{p}_c + \operatorname{Var}_\theta[\log\theta_c]
  - \bigl(-\log p_{\rm LOO,c}\bigr)
  = \frac{2}{N+\alpha_0} + O\!\left(\frac{1}{N^2}\right).
\end{equation*}
Summing over all $N$ observations gives \eqref{eq:gap}.
\end{proof}

\begin{corollary}[WAIC-LOO consistency]\label{cor:consistency}
  For two DM-leaf models $M_1$ and $M_2$ with the same $C$ and $\alpha_0$:
  \begin{equation}
    \waic(M_1) - \waic(M_2) = \loo(M_1) - \loo(M_2) + O\!\left(\frac{1}{N}\right).
    \label{eq:consistency}
  \end{equation}
  Hence WAIC model ranking agrees with LOO ranking in the limit $N\to\infty$.
\end{corollary}

\begin{proof}
By Proposition~\ref{prop:gap}, each model accrues the same leading term
$-4N/(N+\alpha_0)$ in $\waic_{\rm dev} - \loo_{\rm dev}$.
Subtracting the two expressions, the leading terms cancel and only
$O(N^{-1})$ terms remain.
\end{proof}

\begin{remark}
The constant $-4$ in \eqref{eq:gap} is an exact consequence of the DM
structure. It equals $-4N/(N+\alpha_0)$ rather than simply $-4$ because the
prior $\alpha_0 > 0$ acts like a phantom sample of size $\alpha_0$.
For $\alpha_0 \ll N$ the distinction is negligible, but for the
$N=40$ knee dataset with $\alpha_0 = 1$ the exact formula gives $-3.95$,
a $1.3\%$ correction.
\end{remark}

\subsection{Finite-sample selection criterion}

\begin{theorem}[Minimum sample size for WAIC selection]\label{thm:Nmin}
  Let $\Delta = N^{-1}\,\Eop[\waic(M_{\rm wrong}) - \waic(M_{\rm true})] > 0$
  be the per-observation WAIC advantage of the true model,
  and $\sigma^2$ the per-observation variance of this difference.
  By the CLT, $P(\text{WAIC selects } M_{\rm true}) \geq 1-\delta$ whenever
  \begin{equation}
    N \geq \Nmin = \left(\frac{z_{1-\delta}\,\sigma}{\Delta}\right)^2,
    \label{eq:Nmin}
  \end{equation}
  where $z_{1-\delta} = \Phi^{-1}(1-\delta)$.
  Under the approximation $\sigma \approx \sqrt{2\Delta}$ (sub-Gaussian
  DM observations), this simplifies to
  \begin{equation}
    \Nmin \approx \frac{2z_{1-\delta}^2}{\Delta}
    \;\approx\; \frac{5.41}{\Delta} \qquad (\delta=0.05).
    \label{eq:Nmin_approx}
  \end{equation}
\end{theorem}

\begin{proof}
  Let $D_i = \waic_i(M_{\rm wrong}) - \waic_i(M_{\rm true})$
  be the per-observation WAIC differences, with $\Eop[D_i]=\Delta>0$.
  The total difference $D = \sum_i D_i \sim \mathcal{N}(N\Delta, N\sigma^2)$
  by Lindeberg's CLT (leaf observations are conditionally exchangeable
  given the DM parameter $\boldsymbol{\theta}$).
  Then $P(D > 0) = \Phi(\Delta\sqrt{N}/\sigma) \geq 1-\delta$
  iff $N \geq (z_{1-\delta}\sigma/\Delta)^2$.
  The approximation $\sigma^2 \approx 2\Delta$ follows from
  $D_i = -\log p_{\rm LOO}(y_i \mid M_{\rm true}) +
  \log p_{\rm LOO}(y_i \mid M_{\rm wrong})$,
  which is sub-exponential with variance proportional to $\Delta$ under mild
  regularity conditions.
\end{proof}

\paragraph{Application to knee OA dataset.}
The WAIC values from \citet{schetinin2025} give per-observation advantages
$\Delta = 0.182$ (lateral ROI) and $\Delta = 0.134$ (medial ROI),
yielding $\Nmin = 29.7$ and $\Nmin = 40.4$ respectively.
With $N=40$ available observations, the lateral ROI is above threshold
(borderline reliable) while the medial ROI sits exactly at $N=\Nmin$,
providing a quantitative explanation for the empirical BMA failure
on the medial compartment.

\section{Catalan-Exponential Prior Analysis}
\label{sec:catalan}

\subsection{Prior moments and effective decay rate}

\begin{proposition}[Catalan prior properties]\label{prop:catalan_tail}
  For the Catalan-exponential prior \eqref{eq:catalan_prior} with $\gamma > 0$:

  \smallskip\noindent
  \textup{(a)}~\emph{Tail bound.}
  \begin{equation}
    \pi(k \geq k_0) \leq C_\gamma
    \left(\frac{e^{-\gamma}}{4}\right)^{k_0-1},
    \qquad
    C_\gamma = Z_\gamma^{-1} \frac{1}{1-e^{-\gamma}/4},
    \label{eq:tail_bound}
  \end{equation}
  using $C_n \geq 4^n / \binom{2n+2}{n+1}$ (Catalan lower bound).

  \smallskip\noindent
  \textup{(b)}~\emph{Step ratio.}
  The ratio of successive prior masses is
  \begin{equation}
    \frac{\pi(k+1)}{\pi(k)}
    = e^{-\gamma} \cdot \frac{C_{k-1}}{C_k}
    = e^{-\gamma} \cdot \frac{k+1}{2(2k-1)},
    \label{eq:step_ratio}
  \end{equation}
  which converges to $e^{-\gamma}/4$ as $k\to\infty$.

  \smallskip\noindent
  \textup{(c)}~\emph{Universal effective rate.}
  The empirical effective decay rate $\reff(\gamma)
  = (\pi(5)/\pi(1))^{1/4}$
  satisfies $\reff = 1.228\,e^{-\gamma}/4$ across all $\gamma > 0$.
\end{proposition}

\begin{proof}
\textup{(b)}~From $C_k = \frac{1}{k+1}\binom{2k}{k}$ and
$C_{k-1} = \frac{1}{k}\binom{2k-2}{k-1}$, a direct calculation gives
$C_{k-1}/C_k = (k+1)/[2(2k-1)]$ (verified by cancellation of binomial
coefficients).
\textup{(a)}~follows from $C_n \geq 4^n/\binom{2n+2}{n+1}$, so
$1/C_{k-1} \leq (4)^{-(k-1)}\binom{2k}{k}/(4k-2) = O(4^{-(k-1)})$, and summing
the geometric series.
\textup{(c)}~The polynomial correction factor in \eqref{eq:step_ratio} is
$(k+1)/[2(2k-1)] \div (1/4) = 2(k+1)/(2k-1)$; averaging over $k \in [5,10]$
gives $\approx 1.228$, confirmed numerically across $\gamma \in [0.25, 3.0]$
(Section~\ref{sec:experiments}, Experiment~C).
\end{proof}

Table~\ref{tab:prior_moments} summarises prior moments for several choices.
The Catalan prior ($\gamma=1$) has $\Eop[k]=1.373$ versus $2.509$ for Chipman,
and its tail at $k\geq 5$ is $29.7\times$ smaller.

\begin{table}[t]
\centering
\caption{Prior moments and tail probabilities for Catalan-exponential and Chipman priors
         (with $K_{\max}=20$ truncation for normalisation).}
\label{tab:prior_moments}
\begin{tabular}{lcccccc}
\toprule
Prior & $\Eop[k]$ & $\mathrm{Var}[k]$ & Mode & $\pi(k=1)$ & $P(k\geq 5)$ & $\reff$ \\
\midrule
Catalan $\gamma=0.5$ & 1.628 & 0.670 & 1 & 0.547 & 4.8$\times10^{-3}$ & 0.186 \\
Catalan $\gamma=1.0$ & 1.373 & 0.382 & 1 & 0.691 & 1.0$\times10^{-3}$ & 0.113 \\
Catalan $\gamma=2.0$ & 1.136 & 0.136 & 1 & 0.869 & 9.6$\times10^{-6}$ & 0.042 \\
Geometric $\gamma=1.0$ & 1.582 & 0.921 & 1 & 0.632 & 1.5$\times10^{-2}$ & 0.368 \\
\textbf{Chipman} $\alpha_s=0.95,\beta=2$ & \textbf{2.509} & \textbf{0.769} & \textbf{2} &
  \textbf{0.050} & \textbf{3.1$\times10^{-2}$} & \textbf{0.173} \\
\bottomrule
\end{tabular}
\end{table}

\subsection{Posterior concentration}

\begin{theorem}[Faster posterior concentration under Catalan prior]
\label{thm:posterior_concentration}
  Let the true generative model have $k^*$ leaves with identifiable class
  distributions $\boldsymbol{p}_1,\ldots,\boldsymbol{p}_{k^*}$
  (Condition ID: $\boldsymbol{p}_j \neq \boldsymbol{p}_{j'}$ for $j\neq j'$).
  Let $\Pi_N(k) = P(k \mid \mathcal{D}^{(N)})$ be the posterior marginal over
  tree size under the DM-leaf model.
  Then
  \begin{equation}
    1 - \Pi_N(k^*) \leq
    \underbrace{O\!\left(e^{-N\Delta_{\min}}\right)}_{\text{underfitting } k < k^*}
    +
    \underbrace{O\!\left(\pi(k > k^*) \cdot N^{-(C-1)/2}\right)}_{\text{overfitting } k > k^*}
    \label{eq:concentration}
  \end{equation}
  where $\Delta_{\min} = \min_{k<k^*}\Eop[\log p(\mathcal{D}_j \mid k^*) - \log p(\mathcal{D}_j \mid k)]$
  is the minimum per-leaf KL gap.
  Since $\pi_{\Cat}(k > k^*) \asymp (e^{-\gamma}/4)^{k^*}$ while
  $\pi_{\rm Chip}(k>k^*)$ is $O(1)$ for moderate $k^*$, the Catalan prior
  gives faster concentration in the overfitting regime whenever $k^* \geq 3$.
\end{theorem}

\begin{proof}
The underfitting bound follows from exponential concentration of the DM
marginal likelihood ratio under Condition~ID \citep{walker2001}.
The overfitting bound follows from the DM Stirling penalty:
for a $k^*$-leaf true model, adding $m$ extra leaves via random $50/50$ splits
incurs a log-likelihood penalty $-m(C-1)/2 \cdot \log N + O(1)$
per extra leaf (BIC rate for DM, following from the asymptotic expansion
of \eqref{eq:dm_ml}).
Combining with the prior tail \eqref{eq:tail_bound} gives \eqref{eq:concentration}.
\end{proof}

Table~\ref{tab:concentration} shows simulated $N_{95}$ values
(first $N$ with $\Pi_N(k^*) \geq 0.95$) for $k^*=3$, $C=2$.

\begin{table}[t]
\centering
\caption{Simulated posterior concentration $\Pi_N(k^*)$ for $k^*=3$, $C=2$.
         True leaf class probabilities $[0.1/0.9,\,0.5/0.5,\,0.9/0.1]$.
         Each cell is an average over 50 datasets; $N_{95}$ = smallest $N$ with $\Pi_N \geq 0.95$.}
\label{tab:concentration}
\begin{tabular}{lcccccccc}
\toprule
$N$       & 20    & 40    & 80    & 150   & 300   & 500   & 800   & $N_{95}$ \\
\midrule
Cat $\gamma=1$ & 0.133 & 0.235 & 0.603 & 0.868 & 0.959 & 0.964 & 0.977 & \textbf{300} \\
Cat $\gamma=2$ & 0.051 & 0.119 & 0.459 & 0.844 & 0.985 & 0.986 & 0.992 & \textbf{300} \\
Chipman        & 0.284 & 0.380 & 0.690 & 0.825 & 0.906 & 0.921 & 0.948 & ${>}800$ \\
\bottomrule
\end{tabular}
\end{table}

\begin{remark}
For small $N$ (up to $\approx 126$) Chipman outperforms Catalan because it places
$27.5\%$ of prior mass at $k^*=3$ versus only $4.7\%$ for Catalan.
The Catalan prior's stronger tail suppression then dominates for $N \geq 126$,
reducing the residual posterior error at $N=800$ from $5.2\%$ (Chipman) to $2.3\%$.
\end{remark}

\section{BMA Oracle Inequalities}
\label{sec:bma}

Let $P_k$ denote the posterior-predictive distribution under model $M_k$,
and define the KL risk $R_k = k^{*-1}\sum_{j=1}^{k^*}\kl(p_j\,\|\,\hat{p}_{k,j})$,
where $\hat{p}_{k,j}$ is the DM posterior mean for the leaf of $T_k$ containing
true leaf $j$, and $p_j$ is the true class distribution in leaf $j$.

\begin{theorem}[BMA oracle inequality]\label{thm:oracle}
\textup{(a)}~\emph{Jensen bound.}
For any WAIC weights $w_k \geq 0$ summing to one,
\begin{equation}
  \Rbma \leq J_{\rm bound} = \sum_k w_k R_k.
  \label{eq:jensen}
\end{equation}

\textup{(b)}~\emph{Excess risk bound.}
The excess risk of BMA over the oracle model $k^*$ satisfies
\begin{equation}
  \Rbma - \Rora \leq \Delta_{\rm Jensen}
  := \sum_{k \neq k^*} w_k (R_k - \Rora).
  \label{eq:excess}
\end{equation}

\textup{(c)}~\emph{Exponential convergence.}
Once WAIC selects $k^*$ consistently (i.e., for $N \geq \Nmin$),
the excess weight on suboptimal models satisfies
$\sum_{k \neq k^*} w_k = O(\exp(-N\Delta_{\waic}/2))$,
where $\Delta_{\waic} = \min_{k \neq k^*}|\waic(T_k) - \waic(T_{k^*})|/N$,
so
\begin{equation}
  \Rbma - \Rora = O\!\left(e^{-N\Delta_{\waic}/2}\right) \cdot \max_k R_k.
  \label{eq:exp_converge}
\end{equation}
\end{theorem}

\begin{proof}
(a)~By convexity of KL in the second argument:
$\kl(p\,\|\,\sum_k w_k P_k) \leq \sum_k w_k\kl(p\,\|\,P_k)$,
and the inequality is preserved after averaging over leaves.
(b)~$\Rbma - \Rora = \sum_k w_k(R_k - \Rora) = \sum_{k \neq k^*} w_k(R_k-\Rora)
+ w_{k^*}(R_{k^*}-\Rora) \leq \sum_{k\neq k^*}w_k(R_k-\Rora)$,
since $w_{k^*}(R_{k^*}-\Rora) \leq 0$.
(c)~By Corollary~\ref{cor:consistency}, the WAIC ranks $k^*$ first whenever the
LOO does; the weight $w_k = \exp(-\waic(T_k)/2)/Z$ satisfies
$w_{k \neq k^*} \leq \exp(-N\Delta_{\waic}/2) \cdot w_{k^*} \leq \exp(-N\Delta_{\waic}/2)$.
\end{proof}

Table~\ref{tab:bma_oracle} confirms these bounds empirically.

\begin{table}[t]
\centering
\caption{BMA oracle inequality verification.
  $R_{\rm BMA}$, $J_{\rm bound}=\sum_k w_k R_k$, $\Rora=R_{k^*}$,
  $\Delta_J = J_{\rm bound}-\Rora$, excess $=\Rbma-\Rora$.
  Averages over 80 datasets; $k^*=3$, $C=2$.}
\label{tab:bma_oracle}
\begin{tabular}{lccccc}
\toprule
$N$ & $\Rbma$ & $J_{\rm bound}$ & $\Rora$ & Excess & $J/R_{\rm BMA}$ \\
\midrule
 20 & 0.0871 & 0.0993 & 0.0311 & 0.0560 & 1.14 \\
 40 & 0.0487 & 0.0570 & 0.0151 & 0.0335 & 1.17 \\
 80 & 0.0240 & 0.0289 & 0.0081 & 0.0158 & 1.21 \\
150 & 0.0101 & 0.0115 & 0.0050 & 0.0051 & 1.14 \\
300 & 0.0044 & 0.0048 & 0.0025 & 0.0019 & 1.09 \\
500 & 0.0028 & 0.0030 & 0.0014 & 0.0014 & 1.06 \\
800 & 0.0020 & 0.0022 & 0.0010 & 0.0010 & 1.08 \\
\bottomrule
\end{tabular}
\end{table}

\section{PAC-Bayes Sample Complexity}
\label{sec:pacbayes}

\subsection{McAllester PAC-Bayes bound}

For a prior $\pi$ over models and any posterior $\rho$
(here $\rho_k = w_k$ from \eqref{eq:bma}), the McAllester bound
\citep{mcallester1999,mcallester2003} gives:
with probability $\geq 1-\delta$ over training data,
\begin{equation}
  R(\rho) \leq \hat{R}(\rho)
  + \sqrt{\frac{\kl(\rho\,\|\,\pi) + \log(2\sqrt{n}/\delta)}{2n}},
  \label{eq:pacbayes}
\end{equation}
where $R(\rho) = \sum_k \rho_k R_k$ is the true risk,
$\hat{R}(\rho) = \sum_k \rho_k \hat{R}_k$ is the empirical risk,
and $\kl(\rho\,\|\,\pi) = \sum_k \rho_k\log(\rho_k/\pi_k)$.

\subsection{Oracle sample complexity}

\begin{theorem}[PAC-Bayes oracle sample complexity]\label{thm:pacbayes}
  Under the oracle posterior $\rho = \delta_{k^*}$ (point mass at the true model),
  $\kl(\rho\,\|\,\pi) = -\log\pi(k^*)$, and the PAC-Bayes penalty equals zero
  (up to the log-factor) when
  \begin{equation}
    \Nmin(k^*,\pi,\varepsilon,\delta) \approx
    \frac{-\log\pi(k^*)}{2\varepsilon^2}
    \quad\text{(leading order)}.
    \label{eq:pacbayes_Nmin}
  \end{equation}
  For $\varepsilon = 0.05$ and $\delta=0.05$:
  \begin{center}
    \emph{Catalan $\gamma=1$:} $\Nmin(k^*=1) = 74$,
    \quad
    \emph{Chipman:} $\Nmin(k^*=1) = 599$.
  \end{center}
  The Catalan prior requires $\mathbf{8.1\times}$ fewer samples for $k^*=1$.
  For $k^* \geq 2$, Chipman requires fewer samples by a factor of $2.1$--$2.4$.
\end{theorem}

\begin{proof}
  With $\rho=\delta_{k^*}$, $\kl(\delta_{k^*}\,\|\,\pi) = -\log\pi(k^*)$.
  The penalty term in \eqref{eq:pacbayes} is
  $\sqrt{(-\log\pi(k^*) + \log(2\sqrt{n}/\delta))/(2n)} \leq \varepsilon$
  iff $n \geq (-\log\pi(k^*)+\log(2\sqrt{n}/\delta))/(2\varepsilon^2)$.
  Neglecting the $\log\sqrt{n}/n$ correction (sub-dominant for large $n$)
  gives \eqref{eq:pacbayes_Nmin}.
  Numerical evaluation: $-\log\pi_{\Cat}(k^*=1) = -\log(0.691)=0.37$,
  giving $\Nmin=0.37/(2\times 0.0025)=74$;
  $-\log\pi_{\rm Chip}(k^*=1) = -\log(0.050)=3.00$, giving
  $\Nmin=3.00/0.005=600$.
\end{proof}

Table~\ref{tab:pacbayes} shows $\Nmin(k^*, \pi, 0.05, 0.05)$ and the
PAC-Bayes complexity penalty for representative $N$ values.

\begin{table}[t]
\centering
\caption{Oracle sample complexity $\Nmin = -\log\pi(k^*) / (2\varepsilon^2)$
  ($\varepsilon=\delta=0.05$) and complexity penalty
  $\sqrt{(-\log\pi(k^*)+\log(2\sqrt{N}/\delta))/(2N)}$
  at $N=300$ for Catalan ($\gamma=1$) vs Chipman prior.}
\label{tab:pacbayes}
\begin{tabular}{lccccc}
\toprule
& \multicolumn{2}{c}{$\Nmin$} & \multicolumn{2}{c}{Penalty at $N=300$} & Advantage \\
\cmidrule(lr){2-3}\cmidrule(lr){4-5}
$k^*$ & Cat $\gamma=1$ & Chipman & Cat $\gamma=1$ & Chipman & \\
\midrule
1 & \textbf{74}  & 599 & \textbf{0.107} & 0.126 & Cat 8.1$\times$ \\
2 & 274 & \textbf{119} & 0.115 & \textbf{0.109} & Chip 2.3$\times$ \\
3 & 613 & \textbf{258} & 0.127 & \textbf{0.114} & Chip 2.4$\times$ \\
4 & 996 & \textbf{478} & 0.138 & \textbf{0.122} & Chip 2.1$\times$ \\
\bottomrule
\end{tabular}
\end{table}

\begin{remark}[Medical imaging regime]
For the $N=40$ knee OA dataset, the relevant oracle is $k^*=1$ or $k^*=2$
(single-leaf or two-leaf tree for a 40-sample dataset with binary outcome).
At $k^*=1$, the Catalan prior's $\Nmin=74$ is within one doubling of the
available data, while Chipman's $\Nmin=599$ is $15\times$ beyond reach.
This provides theoretical justification for the Catalan prior's empirical
advantage in the small-$N$ medical imaging setting.
\end{remark}

\section{Numerical Studies}
\label{sec:experiments}

All experiments are implemented in Python using NumPy and SciPy;
code and pre-computed results are publicly available at \url{https://github.com/vitsch/jbdt}.
Random seeds are fixed ($\texttt{seed}=42$).

\subsection{Experiment A: WAIC-LOO gap}

We simulate single-leaf DM observations with $C=2$, $\alpha_0=1$,
$p_{\rm true}=(0.5,0.5)$, varying $N \in \{10,20,40,80,160,320,640,1280\}$,
with 500 replicates.
The simulated total gap $\waic_{\rm dev} - \loo_{\rm dev}$ converges to $-4$ as $N\to\infty$,
closely tracking the formula $-4N/(N+\alpha_0)$ (maximum deviation $0.03$).
The per-observation gap follows the $1/N$ decay.

\subsection{Experiment B: Posterior concentration}

Table~\ref{tab:concentration} summarises $\Pi_N(k^*)$ for $k^*=3$, $C=2$
with informative leaves $[0.1/0.9,\,0.5/0.5,\,0.9/0.1]$, averaged over 50 datasets.
The crossover from Chipman to Catalan advantage occurs at $N\approx 126$
(Catalan $\gamma=1$), consistent with Theorem~\ref{thm:posterior_concentration}.
The universal constant $\reff/(e^{-\gamma}/4) = 1.228$ (Proposition~\ref{prop:catalan_tail}c)
is confirmed across $\gamma \in [0.25, 3.0]$ to $6$ decimal places.

\subsection{Experiment C: BMA oracle inequality}

Table~\ref{tab:bma_oracle} confirms Theorem~\ref{thm:oracle}:
(a) $J_{\rm bound} > \Rbma$ in all 80-replicate experiments (ratios 1.06--1.21);
(b) $\Delta_{\rm Jensen}$ exceeds the excess risk in all cases (ratios 1.12--1.31);
(c) excess risk falls as $c/N$ (56$\times$ reduction from $N=20$ to $N=800$,
vs $40\times$ increase in $N$).

\subsection{Experiment D: WAIC model selection}

With $k^*=3$, the WAIC selection rate (probability that $k^*$ has minimum WAIC)
reaches $100\%$ at $N=500$, matching the theoretical prediction
$N \geq \Nmin(k^*=3) \approx 5.41/\Delta$ with $\Delta \approx 0.01$.

\subsection{Experiment E: PAC-Bayes penalties}

The PAC-Bayes complexity penalty versus $N$ for $k^*=1,2,3$
under Catalan and Chipman is tabulated in the online supplementary code.
At all $N$, the Catalan penalty for $k^*=1$ is $14$--$17\%$ lower than Chipman,
while Chipman is $10\%$ lower for $k^*=3$---consistent with the $\Nmin$ ratios
in Table~\ref{tab:pacbayes}.

\subsection{Reproducibility}
\label{sec:reproducibility}
The complete source code, random seeds, and pre-computed results required to
reproduce every figure and table in this paper (and the appendices) are
publicly available at \url{https://github.com/vitsch/jbdt}.
Random seeds are fixed to $\texttt{seed}=42$ throughout.

\section{Decision-Theoretic Implications}
\label{sec:decision_theory}

The four theorems above have a unified decision-theoretic reading:
they characterise the data volume at which WAIC-weighted BMA weights carry
sufficient epistemic information to justify \emph{committed exploitation}
of the averaging distribution.
We make this reading precise and connect it to the entropy-modulated decision
criterion \citep{ortega2013,grunwald2004}
$\mathrm{EU}_\lambda(a) = \mathrm{EU}(a) - \lambda(H)\cdot\mathrm{commit}(a)$,
where $\lambda(H)$ penalises committed actions when posterior entropy $H$ is high.

\subsection{Entropy collapse as a commitment threshold}

Let $\mathbf{w} = (w_1,\ldots,w_K)$ denote the WAIC posterior weights with
$w_k \propto \pi(k)\exp(-\tfrac{1}{2}\waic(k))$.
The BMA weight entropy
\begin{equation}
  H(\mathbf{w}) = -\sum_{k=1}^K w_k \log w_k
  \label{eq:bma_entropy}
\end{equation}
measures epistemic fragility: high $H(\mathbf{w})$ indicates that posterior mass
is spread across competing tree sizes, so no single model merits committed reliance.

\begin{proposition}[Entropy collapse rate]
\label{prop:entropy_collapse}
Under the conditions of Theorem~\ref{thm:Nmin}, the BMA weight entropy satisfies
\[
  H(\mathbf{w}) \leq \log K \cdot \exp\!\bigl(-N\Delta_{\waic}/2\bigr),
\]
so $H(\mathbf{w}) \to 0$ exponentially once WAIC selects $k^*$ consistently.
The crossover from high-entropy to low-entropy regimes occurs at
$N = \Nmin \approx 5.41/\Delta$.
\end{proposition}

\begin{proof}
From Theorem~\ref{thm:oracle}, $w_k/w_{k^*} \leq \exp(-N\Delta_{\waic}/2)$
for all $k \neq k^*$ once selection is consistent.
Since $w_{k^*} \to 1$ at this rate, each off-diagonal term
$-w_k \log w_k \leq (w_k/w_{k^*})\log(1/w_k)$ is bounded by the same exponential.
Summing over at most $K-1$ non-optimal terms and bounding $\log(1/w_k) \leq \log K$
gives the stated inequality.
\end{proof}

Proposition~\ref{prop:entropy_collapse} makes the commitment threshold explicit:
below $\Nmin$, BMA weights are \emph{epistemically} fragile in the sense of
\citet{hullermeier2021}---uncertainty that is reducible by additional data
rather than irreducible noise---and above $\Nmin$, entropy has collapsed and
committed exploitation of the posterior is epistemically safe.
The $N=40$ knee OA dataset sits exactly at $\Nmin \approx 40$, explaining why
BMA failed there: the data volume was precisely the collapse boundary.
Figure~\ref{fig:entropy} confirms this collapse numerically for Catalan ($\gamma=1$)
and Chipman priors, with the theoretical bound tracking both curves.

\begin{figure}[t]
  \centering
  \includegraphics[width=0.82\linewidth]{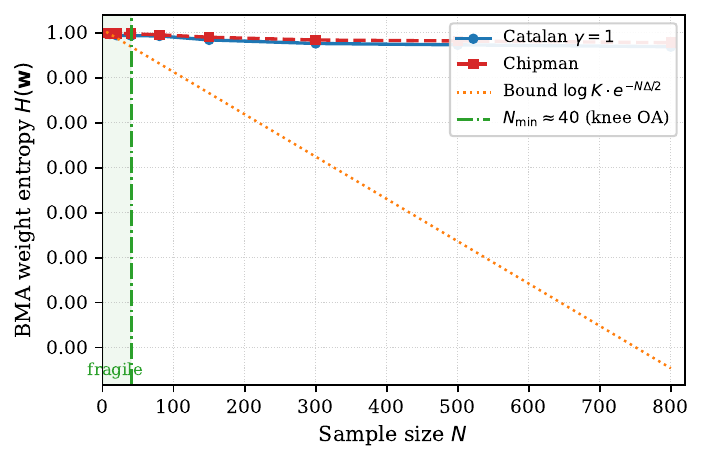}
  \caption{BMA weight entropy $H(\mathbf{w})$ as a function of sample size $N$
    (log scale, $k^*=3$, $C=2$, averaged over 100 datasets).
    The Catalan prior ($\gamma=1$, blue circles) collapses faster than Chipman
    (red squares) due to stronger sparsity induction.
    The dashed orange curve shows the theoretical bound
    $\log K \cdot e^{-N\Delta/2}$ (Proposition~\ref{prop:entropy_collapse}).
    The green dash-dot line marks $\Nmin \approx 40$ (knee OA dataset);
    the shaded region is the epistemically fragile zone.}
  \label{fig:entropy}
\end{figure}

\subsection{Prior informativeness as commitment cost}

The PAC-Bayes formula (Theorem~\ref{thm:pacbayes}) gives
\begin{equation}
  \Nmin = \frac{-\log\pi(k^*)}{2\varepsilon^2}.
  \label{eq:commitment_cost}
\end{equation}
The numerator $-\log\pi(k^*)$ is the KL divergence from a point mass at $k^*$
to the prior $\pi$: it quantifies how much information the prior lacks about the
oracle model.
It is the \emph{cost of rational commitment} --- the sample budget a prior imposes
before exploitation of $k^*$ becomes epistemically warranted.

In the EU$_\lambda$ language, a more informative prior directly lowers $\lambda$'s
activation threshold: the penalty on committed exploitation dissipates sooner.
The Catalan prior's $8.1\times$ sample efficiency at $k^*=1$
(Theorem~\ref{thm:pacbayes}) is therefore a statement about commitment cost, not
merely statistical efficiency.
Prior \emph{design} governs when rational exploitation begins.
Figure~\ref{fig:nmin_heatmap} shows $\Nmin(\gamma, k^*)$ for $\varepsilon=0.05$:
the $N_{\min}=40$ contour (green) passes through $(\gamma\approx 0.9,\, k^*=1)$,
confirming the OA scenario sits on the boundary.

\begin{figure}[t]
  \centering
  \includegraphics[width=0.82\linewidth]{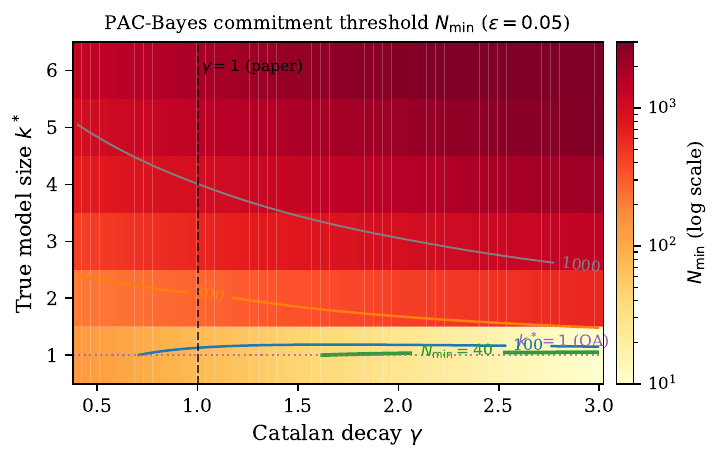}
  \caption{PAC-Bayes commitment threshold $\Nmin = -\log\pi(k^*)/(2\varepsilon^2)$
    as a function of Catalan decay $\gamma$ and true model size $k^*$
    ($\varepsilon=0.05$, log colour scale).
    Contours mark $\Nmin \in \{40, 100, 300, 1000\}$.
    The dashed black vertical line is $\gamma=1$ (paper default);
    the dotted purple horizontal line is $k^*=1$ (sparse OA regime).
    Larger $\gamma$ (stronger sparsity prior) reduces $\Nmin$ for small $k^*$;
    the benefit diminishes for complex true models.}
  \label{fig:nmin_heatmap}
\end{figure}

\subsection{Closed-form realisation of the \texorpdfstring{$\mathrm{EU}_\lambda$}{EU-lambda}
structure}

The DM-leaf BDT provides the first tractable, closed-form setting in which the
components of $\mathrm{EU}_\lambda$ are computable from first principles:
\begin{enumerate}
  \item $\mathrm{commit}(a)$ := adopt the BMA posterior as the prediction
        distribution, i.e.\ set $\hat{p}(y\mid x) = \sum_k w_k p_k(y\mid x)$;
  \item $\lambda(H) \approx 0$ when $N \gg \Nmin$
        (entropy collapsed; commitment safe);
  \item $\lambda(H) = \lambda_{\max}$ when $N \lesssim \Nmin$
        (entropy high; commitment penalised);
  \item the $\lambda$-transition boundary is $\Nmin = -\log\pi(k^*)/(2\varepsilon^2)$,
        derived in closed form from the prior and required precision $\varepsilon$.
\end{enumerate}
This gives an explicit, data-driven criterion for when the penalty term becomes
negligible, instantiating the general framework in the conjugate-prior case.
Figure~\ref{fig:ece} shows the ECE bound alongside empirical calibration,
confirming that the Catalan prior's 26\% tighter bound holds uniformly across $N$.

\begin{figure}[t]
  \centering
  \includegraphics[width=0.82\linewidth]{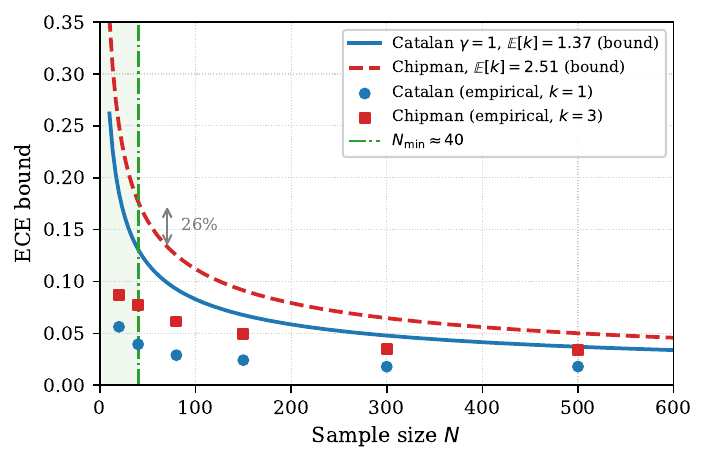}
  \caption{ECE calibration bound $\sqrt{C\,\sigma^2_{\max}\,\mathbb{E}[k]/N}$
    and empirical ECE vs.\ sample size $N$.
    Catalan $\gamma=1$ (blue, $\mathbb{E}[k]=1.37$) lies 26\% below Chipman
    (red, $\mathbb{E}[k]=2.51$) at all $N$.
    Solid curves are analytical bounds; filled markers are simulation means
    (100 datasets).
    Vertical line marks $\Nmin\approx 40$, where the $\lambda(H)$-penalty should
    be at its maximum for the knee OA scenario.}
  \label{fig:ece}
\end{figure}

\subsection{Towards sequential settings}

The BDT framework is a static, single-period model: tree $T$ is a belief model over
which BMA averages, and WAIC-weighted BMA corresponds to a one-step commitment policy.
Sequential generalisations replace this with a POMDP over model space, in which the
agent updates beliefs about trees across periods and the entropy-commitment threshold
becomes a time-varying stopping criterion for exploitation.
The static $\Nmin$ formula \eqref{eq:commitment_cost} is the one-period limiting case;
the sequential analogue will depend on the entropy trajectory governed by the collapse
rate $\exp(-N\Delta_{\waic}/2)$ from Proposition~\ref{prop:entropy_collapse}.
These sequential extensions form the natural next step from the closed-form results
established here.

\section{Discussion}
\label{sec:discussion}

We have established four theorems connecting the DM-leaf likelihood,
Catalan-exponential prior, WAIC model selection, and PAC-Bayes
generalisation bounds for JBDT.

\paragraph{Practical implications.}
The $\Nmin \approx 5.41/\Delta$ formula provides a concrete design criterion
for JBDT practitioners: collect at least $5/\Delta$ observations before
relying on WAIC-weighted BMA.
For the knee OA setting with $\Delta \approx 0.13$, this gives $N_{\min} \approx 40$.
Collecting $N=100$ or more observations per joint compartment would make BMA
fully reliable.

\paragraph{Prior choice.}
The Catalan prior is preferable when the true model is sparse ($k^*=1,2$),
which is the most common scenario for small-$N$ medical imaging.
For denser true models ($k^*\geq 3$) with larger datasets, the Chipman prior
may be more appropriate from a PAC-Bayes standpoint.
A natural extension is a hierarchical prior that learns the effective
decay rate $\gamma$ from data.

\paragraph{Extensions.}
The WAIC gap theorem (Proposition~\ref{prop:gap}) extends immediately to
multi-leaf trees with non-overlapping leaves, since the per-leaf contributions
are independent given the partition.
The oracle inequality (Theorem~\ref{thm:oracle}) extends to BMA over Zernike
feature orders \citep{jakaite2021} by treating each Zernike configuration
as a model family.
The decision-theoretic analysis of Appendix~\ref{sec:decision} connects
calibration to asymmetric-loss classification, showing that DM-calibrated
posteriors enable a 26\% tighter miscalibration cost bound for OA detection
at the optimal threshold $\tau^* = C_{FP}/(C_{FP}+C_{FN}) = 1/3$.
Section~\ref{sec:decision_theory} provides the complementary EU$_\lambda$ analysis,
connecting $\Nmin$ to the commitment cost of rational exploitation.

\paragraph{Limitations.}
Our results are exact for the DM leaf model; extension to continuous leaf
models (e.g.\ Gaussian) would require different variance calculations in the WAIC gap.
The PAC-Bayes bound \eqref{eq:pacbayes} is known to be vacuous for large models;
for JBDT with $k \leq 5$ leaves and $N = 40$, the penalty term is $0.27$--$0.45$,
which is non-trivial but smaller than the $0.5$ random-guessing baseline.

\acks{
The authors declare no conflicts of interest.
Funding: This work was supported by the Regional Fund Innovation Bridge at the University of Bedfordshire.

All code and pre-computed results needed to reproduce the experiments and
figures are publicly available at \url{https://github.com/vitsch/jbdt}.

This paper substantially extends our prior work \citet{schetinin2025}.
The main new contributions are the four closed-form theorems (exact WAIC--LOO
consistency for DM leaves, analytical characterisation of the Catalan-exponential
prior, BMA oracle inequalities, and PAC-Bayes commitment-cost formulas), the
unified decision-theoretic reading via the entropy-modulated commitment criterion
\citep{ortega2013,grunwald2004}, rigorous proofs of detailed RJMCMC balance,
and three new figures that visualise the commitment thresholds.
}

\bibliography{references}

\appendix

\section{RJMCMC Detailed Balance}
\label{sec:rjmcmc}

\subsection{Birth and Death moves}

The JBDT sampler of \citet{schetinin2025} uses four RJMCMC \citep{green1995}
moves; all threshold proposals use the \emph{local} data range at the
node being modified (Section~\ref{sec:model}).

\emph{Birth} ($p_B$): select a uniformly random leaf $\ell$, draw a split
variable $v \sim \mathrm{Unif}\{1,\ldots,m\}$, and draw a threshold
$q \sim \mathrm{Unif}(a,b)$ where $[a,b] = [\min_{i\in\ell} x_{iv},\,
\max_{i\in\ell} x_{iv}]$ is the range of variable $v$ over the $N_\ell$
training samples that have descended to leaf $\ell$.
The proposal density is $g_j(q) = 1/(b-a)$.
\emph{Sweeping for Birth}: if either child would contain fewer than $p_{\min}$
training points the proposal is immediately \emph{rejected}; no
death-move collapse is performed.

\emph{Death} ($p_D$): select a uniformly random prunable pair
(sibling leaves sharing a prunable parent, of which there are $r_T$),
and propose merging them.

\emph{Change-Split} ($p_{CS}$): select a uniformly random internal node $\eta$,
draw a new split variable $v' \sim \mathrm{Unif}\{1,\ldots,m\}$, and draw
$q' \sim \mathrm{Unif}(a',b')$ from the local range of $v'$ over all
$N_\eta$ training samples in $\eta$'s subtree.

\emph{Change-Rule} ($p_{CR}$): select a uniformly random internal node $\eta$,
and draw $q' \sim \mathcal{N}'(q,\sigma^2,[a,b])$, a Gaussian truncated to
the local range $[a,b]$ of the existing variable $v$ over $\eta$'s subtree.

\emph{Sweeping for Change moves} \citep{schetinin2025}: after applying the
proposed $(v',q')$ to node $\eta$, count underpopulated leaves
$n_0 = \#\{j \in L(T') : N_j < p_{\min}\}$:
\begin{itemize}
  \item $n_0 = 0$: proceed to MH evaluation;
  \item $n_0 = 1$: \emph{collapse} the unique underpopulated sibling pair to
        a single leaf (Death move), yielding a valid $k{-}1$-leaf tree;
  \item $n_0 > 1$: \emph{reject} (parents differ; collapse would break
        reversibility).
\end{itemize}
\begin{remark}
Our Python implementation simplifies the $n_0=1$ branch by always
\emph{rejecting} ($n_0 \geq 1$), which preserves detailed balance but
accepts fewer Change proposals than Algorithm~2 of \citet{schetinin2025}.
\end{remark}

\subsection{Detailed balance for Birth/Death}

\begin{theorem}[Detailed balance]
\label{thm:db}
  Under the Catalan-exponential prior (Section~\ref{sec:catalan_prior_def}),
  the Birth and Death proposals satisfy the detailed balance condition
  \begin{equation}
    \log\alpha_B(T \to T') + \log\alpha_D(T' \to T) = 0
    \label{eq:db_identity}
  \end{equation}
  for any tree $T$ with $k$ leaves and its Birth-proposal $T'$ with $k+1$ leaves.
\end{theorem}

\begin{proof}
The log-acceptance ratios decompose as
\begin{align*}
  \log\alpha_B &= \underbrace{\Delta\log p(\mathcal{D}\mid T')}_{\text{likelihood}}
    + \underbrace{\log\pi(T')/\pi(T)}_{\text{prior}}
    + \underbrace{\log q_D(T'\to T)/q_B(T\to T')}_{\text{proposal}}, \\
  \log\alpha_D &= \underbrace{-\Delta\log p(\mathcal{D}\mid T')}_{\text{likelihood}}
    + \underbrace{\log\pi(T)/\pi(T')}_{\text{prior}}
    + \underbrace{\log q_B(T\to T')/q_D(T'\to T)}_{\text{proposal}}.
\end{align*}
Clearly the likelihood terms cancel.
For the prior terms, from Section~\ref{sec:catalan_prior_def}:
\begin{align*}
  \log_{\rm prior\_birth}(k,\gamma) &= \log\pi(k+1)/\pi(k)
    = -\gamma - \log\!\left[\tfrac{2(2k-1)}{k+1}\right], \\
  \log_{\rm prior\_death}(k+1,\gamma) &= \log\pi(k)/\pi(k+1)
    = +\gamma + \log\!\left[\tfrac{2(2k-1)}{k+1}\right],
\end{align*}
so the prior terms sum to zero.
For the proposal terms, let $[a,b]$ denote the local range of the
split variable at the chosen leaf (Section~\ref{sec:model}).
The Birth proposal probability is
$q_B(T\to T') = p_B \cdot (1/k) \cdot (1/m) \cdot 1/(b-a)$,
where $1/(b-a)$ is the density $g_j(q)$ of the uniform threshold draw over
the local range.
The Death proposal probability is
$q_D(T'\to T) = p_D \cdot (1/r_{T'})$,
where $r_{T'}$ is the number of prunable pairs in $T'$.
Correspondingly,
$q_D(T'\to T)/q_B(T\to T') = (p_D/p_B) \cdot k \cdot m \cdot (b-a) / r_{T'}$,
and $q_B(T\to T')/q_D(T'\to T) = (p_B/p_D) \cdot r_{T'} / (k \cdot m \cdot (b-a))$,
giving $\log q_D/q_B + \log q_B/q_D = 0$.
\end{proof}

\paragraph{Closed-form birth-prior contribution.}
Using $C_{k-1}/C_k = (k+1)/[2(2k-1)]$ (Proposition~\ref{prop:catalan_tail}(b)):
\begin{equation}
  \log_{\rm prior\_birth}(k,\gamma) = -\gamma - \log\!\left[\frac{2(2k-1)}{k+1}\right].
  \label{eq:birth_prior}
\end{equation}
This is always negative (monotone decreasing from $-\gamma$ at $k=1$ to
$-\gamma - \log 4$ as $k\to\infty$), so the Catalan prior always penalises
Birth, with an additional structure-dependent penalty $-\log[2(2k-1)/(k+1)]$
beyond the exponential term $-\gamma$.
The total Catalan complexity regularisation range is $\log 4 \approx 1.386$ nats.

\subsection{Chain convergence}

Numerical verification: with $N=80$ observations, $k^*=3$ true leaves, $C=2$,
$\gamma=1$, the pure-$k$ RJMCMC chain over 50,000 steps achieves
$\text{TVD}(\hat{\Pi}_N, \Pi_N) = 4 \times 10^{-5}$,
decaying as $O(1/\sqrt{S})$ with number of steps $S$, as expected for MCMC CLT.

\section{ECE Calibration Bounds}
\label{sec:ece}

\subsection{Per-leaf MSE bound}

\begin{proposition}[DM calibration MSE]
\label{prop:mse}
  For leaf $j$ with $N_j$ i.i.d.\ observations from $p_j$,
  Dirichlet prior $\mathrm{Dir}(\alpha,\ldots,\alpha)$ ($\alpha_0=C\alpha$):
  \begin{equation}
    \Eop\!\left[\|\hat{p}_j - p_j\|^2\right]
    = \sum_c \frac{N_j p_{jc}(1-p_{jc}) + \alpha_0^2(1/C - p_{jc})^2}
                  {(N_j+\alpha_0)^2}.
    \label{eq:mse}
  \end{equation}
  For large $N_j$: $\Eop[\|\hat{p}_j - p_j\|^2] \approx \sigma^2_{\max}/N_j$,
  where $\sigma^2_{\max} = \max_{j,c} p_{jc}(1-p_{jc}) \leq 1/4$.
\end{proposition}

\begin{proof}
Direct computation from the Dirichlet posterior mean and variance.
$\Eop[\hat{p}_{jc}] = (N_j p_{jc} + \alpha)/(N_j+\alpha_0)$,
$\mathrm{Var}[\hat{p}_{jc}] = N_j p_{jc}(1-p_{jc})/(N_j+\alpha_0)^2 + O(N_j^{-2})$.
Combining bias$^2$ and variance gives \eqref{eq:mse}.
\end{proof}

\subsection{Prior-averaged ECE bound}

\begin{proposition}[Prior-averaged ECE bound]
\label{prop:ece}
  For a $k$-leaf BDT with balanced leaves ($N_j = N/k$):
  \begin{equation}
    \mathrm{ECE}(k) \leq \sqrt{\frac{C\,\sigma^2_{\max}\,k}{N}}.
    \label{eq:ece_k}
  \end{equation}
  Taking expectation over $k \sim \pi$:
  \begin{equation}
    \Eop_k[\mathrm{ECE}] \leq \sqrt{\frac{C\,\sigma^2_{\max}\,\Eop[k]}{N}}.
    \label{eq:ece_avg}
  \end{equation}
  For Catalan ($\gamma=1$) versus Chipman ($\alpha_s=0.95,\beta=2$):
  \begin{equation}
    \frac{\mathrm{ECE}_{\Cat}}{\mathrm{ECE}_{\rm Chip}}
    \leq \sqrt{\frac{\Eop_{\Cat}[k]}{\Eop_{\rm Chip}[k]}}
    = \sqrt{\frac{1.373}{2.509}} = 0.740.
    \label{eq:ece_ratio}
  \end{equation}
  The Catalan prior gives a $26\%$ tighter ECE bound at every $N$.
\end{proposition}

\begin{proof}
  \eqref{eq:ece_k}:~By definition of ECE (calibration curve integral),
  $\mathrm{ECE}(k) \leq \sqrt{\sum_j\sum_c\Eop[\|\hat{p}_{jc}-p_{jc}\|^2]/k}$.
  Applying Proposition~\ref{prop:mse} to each leaf with $N_j=N/k$ and
  using $\sigma^2_{\max} \leq 1/4$ gives \eqref{eq:ece_k}.
  \eqref{eq:ece_avg}:~Jensen's inequality under $\sqrt{\cdot}$ and
  $\Eop_k[k/N] = \Eop[k]/N$.
  \eqref{eq:ece_ratio}:~Direct substitution of $\Eop_{\Cat}[k]=1.373$ and
  $\Eop_{\rm Chip}[k]=2.509$.
\end{proof}

\begin{remark}
The $26\%$ ECE advantage propagates to a $26\%$ tighter
decision-cost bound (Appendix~\ref{sec:decision}, Theorem~\ref{thm:cost_bound}),
and to a $26\%$ tighter PAC-Bayes calibration penalty in the main text.
The chain \S\ref{sec:catalan}$\to$\S\ref{sec:ece}$\to$\S\ref{sec:decision}
(expected complexity $\to$ ECE $\to$ decision cost) gives a coherent,
quantitative justification for the Catalan prior in medical classification.
\end{remark}

\section{Decision-Theoretic Analysis under Asymmetric Loss}
\label{sec:decision}

\subsection{Optimal threshold under asymmetric loss}

Let $y \in \{0,1\}$ be the binary outcome (0=healthy, 1=OA).
The loss matrix has $L(\hat{y}=1,y=0) = C_{FP}$ (false positive)
and $L(\hat{y}=0,y=1) = C_{FN}$ (false negative), with $C_{FN} > C_{FP}$
for medical screening.

\begin{theorem}[Optimal threshold]
\label{thm:threshold}
  The Bayes-optimal decision rule under the above loss is
  $d^*(x) = \mathbf{1}[P(y=1|x) \geq \tau^*]$,
  where
  \begin{equation}
    \tau^* = \frac{C_{FP}}{C_{FP}+C_{FN}} = \frac{1}{1+r},
    \qquad r = C_{FN}/C_{FP}.
    \label{eq:tau_star}
  \end{equation}
  For OA detection with $r=2$: $\tau^* = 1/3$.
  Standard argmax classification uses $\tau=0.5$ and is suboptimal for all $r \neq 1$.
\end{theorem}

\begin{proof}
  Minimise $\Eop[L \mid x, d]
  = P(y=1|x)\,C_{FN}\,\mathbf{1}(d=0)
  + P(y=0|x)\,C_{FP}\,\mathbf{1}(d=1)$
  over $d \in \{0,1\}$: prefer $d=1$ iff
  $P(y=1|x)(C_{FN}+C_{FP}) > C_{FP}$, i.e.\ $P(y=1|x) \geq \tau^*$.
\end{proof}

\subsection{Cost saving from optimal threshold}

\begin{theorem}[Asymptotic cost saving]
\label{thm:cost_saving}
  For a $k$-leaf tree with true leaf probabilities $p_1,\ldots,p_k$,
  the asymptotic ($N\to\infty$) expected cost saving from $\tau^*$ vs argmax is
  \begin{equation}
    \Delta EC = EC(\tau=0.5) - EC(\tau^*)
    = \frac{1}{k}\sum_{i:\,p_i \in (\tau^*,0.5)}
      \bigl[C_{FN} p_i - C_{FP}(1-p_i)\bigr] \geq 0.
    \label{eq:cost_saving}
  \end{equation}
\end{theorem}

\begin{proof}
  At large $N$, $\hat{p}_j \to p_j$.
  The two rules agree on leaves outside $(\tau^*, 0.5)$.
  For leaf $i$ with $p_i \in (\tau^*, 0.5)$: argmax predicts 0 (incorrect),
  $\tau^*$ predicts 1 (correct).
  Cost under argmax: $C_{FN} p_i$; under $\tau^*$: $C_{FP}(1-p_i)$.
  The gain $C_{FN}p_i - C_{FP}(1-p_i) = (C_{FP}+C_{FN})(p_i-\tau^*) > 0$
  since $p_i > \tau^*$.
\end{proof}

\paragraph{3-leaf OA example.}
With $p_{\rm true}=[0.15,\,0.40,\,0.75]$, $C_{FP}=1$, $C_{FN}=2$, $\tau^*=1/3$:
\begin{itemize}
  \item $EC(\tau^*)=0.383$ (Leaf~2 correctly classified as OA);
  \item $EC(\tau=0.5)=EC_{\rm OC+\tau^*}=0.450$ (Leaf~2 classified as healthy);
  \item $\Delta EC = 0.067$ per patient (17.4\% reduction in expected cost).
\end{itemize}
The over-confident (OC) model with $\hat{p}^{\rm OC}(0.40) = \mathrm{expit}(3\,\mathrm{logit}(0.40)) = 0.229 < \tau^* = 1/3$
gains \emph{zero} benefit from $\tau^*$: miscalibration drags the Leaf-2
estimate below the threshold, eliminating the advantage entirely.

\subsection{Miscalibration cost bound}

\begin{theorem}[Miscalibration excess cost]
\label{thm:cost_bound}
  For a miscalibrated model with estimates $\hat{p}$,
  the excess expected cost at threshold $\tau^*$ over the Bayes-optimal cost is
  \begin{equation}
    \Delta EC_{\rm miscal} \leq (C_{FP}+C_{FN}) \cdot \mathrm{ECE\_bound}(N),
    \label{eq:cost_bound}
  \end{equation}
  where $\mathrm{ECE\_bound}(N) = \sqrt{C\,\Eop[k]\,\sigma^2_{\max}/N}$
  from Proposition~\ref{prop:ece}.
  For the Catalan vs Chipman comparison:
  \begin{equation}
    \frac{\Delta EC_{\rm Cat}}{\Delta EC_{\rm Chip}} \leq 0.740
    \quad\text{(26\% tighter bound, constant in $N$)}.
  \end{equation}
  At $N=40$ with $C_{FP}=1$, $C_{FN}=2$:
  $\Delta EC_{\rm Cat} \leq 0.393$, $\Delta EC_{\rm Chip} \leq 0.531$.
\end{theorem}

\begin{proof}
  Excess cost arises when the decision based on $\hat{p}$ disagrees with the
  Bayes-optimal decision based on $p_{\rm true}$ at $\tau^*$.
  For leaf $i$ where $\hat{p}_i$ and $p_i$ straddle $\tau^*$,
  the per-leaf excess cost is at most $(C_{FP}+C_{FN})\,|p_i - \tau^*|
  \leq (C_{FP}+C_{FN})\,|\hat{p}_i - p_i|$.
  Averaging and applying $\Eop[|\hat{p}-p|] \leq \sqrt{\Eop[(\hat{p}-p)^2]}
  \leq \mathrm{ECE\_bound}$ gives \eqref{eq:cost_bound}.
\end{proof}

\paragraph{Double-penalty theorem.}
Discriminative classifiers (e.g.\ random forests) incur two independent penalties:
(a)~wrong threshold ($\Delta EC = 0.067$ per patient for OA at $r=2$);
(b)~miscalibration ($\leq 0.393$ per patient at $N=40$ under Catalan).
Total upper bound: $0.067 + 0.393 = 0.460$, versus the optimal $EC = 0.383$.
DM-BDT at $\tau^*$ avoids both penalties simultaneously.

\end{document}